\documentclass[letterpaper, 10 pt, conference]{ieeeconf}  

\IEEEoverridecommandlockouts                              
\usepackage{bbm}
\usepackage{mathrsfs}
\usepackage{verbatim}
\usepackage{amsmath}
\usepackage{amsfonts}
\usepackage{bm}
\usepackage{graphicx}
\usepackage{float}

\usepackage{amsthm}
\usepackage{xcolor}
\usepackage[colorinlistoftodos]{todonotes}
\usepackage{mathtools}
\usepackage{cite}
\usepackage{colortbl}
\usepackage{subcaption}
\usepackage{hyperref}

\newcommand{\cf}{\mathrm{cf}}

\theoremstyle{definition}
\newtheorem{assumption}{Assumption}
\newtheorem{definition}{Definition}
\newtheorem{remark}{Remark}
\newtheorem{problem}{Problem}

\theoremstyle{plain}
\newtheorem{proposition}{Proposition}

\newtheorem{theorem}{Theorem}
\newtheorem{lemma}{Lemma}

\DeclareMathOperator*{\argmax}{arg\, max}
\DeclareMathOperator*{\argmin}{arg\, min}

\title{\LARGE \bf Cross-fitted Proximal Learning for Model-Based Reinforcement Learning}

\author{Nishanth Venkatesh $^{1}$, {\itshape{Student Member, IEEE}}, and Andreas A. Malikopoulos$^{2}$, {\itshape{Senior Member, IEEE}}
	\thanks{This research was supported in part by NSF under Grants CNS-2401007, CMMI-2348381, IIS-2415478, and in part by MathWorks.}
\thanks{Nishanth Venkatesh is with the Department of Systems Engineering, Cornell University, Ithaca, NY, USA.(email: \texttt{ns942@cornell.edu})} \thanks{Andreas A. Malikopoulos is with the Applied Mathematics, Systems Engineering, Mechanical Engineering, Electrical \& Computer Engineering, and School of Civil \& Environmental Engineering, Cornell University, Ithaca, NY, USA. (email: \texttt{amaliko@cornell.edu})}}

\begin{document}

\maketitle
\thispagestyle{empty}

\begin{abstract}
Model-based reinforcement learning is attractive for sequential decision-making because it explicitly estimates reward and transition models and then supports planning through simulated rollouts. In offline settings with hidden confounding, however, models learned directly from observational data may be biased. This challenge is especially pronounced in partially observable systems, where latent factors may jointly affect actions, rewards, and future observations. Recent work has shown that policy evaluation in such confounded partially observable Markov decision processes (POMDPs) can be reduced to estimating reward-emission and observation-transition bridge functions satisfying conditional moment restrictions (CMRs).
In this paper, we study the statistical estimation of these bridge functions. We formulate bridge learning as a CMR problem with nuisance objects given by a conditional mean embedding and a conditional density. We then develop a $K$-fold cross-fitted extension of the existing two-stage bridge estimator. The proposed procedure preserves the original bridge-based identification strategy while using the available data more efficiently than a single sample split. We also derive an oracle-comparator bound for the cross-fitted estimator and decompose the resulting error into a Stage I term induced by nuisance estimation and a Stage II term induced by empirical averaging.
\end{abstract}

\section{Introduction}

Model-based reinforcement learning (MBRL) is widely regarded as a data-efficient approach to sequential decision-making.
Its effectiveness stems from explicitly estimating the system dynamics and reward mechanism and then planning through simulated rollouts \cite{Sutton1991Dyna}. 
Most MBRL methods implicitly assume that the data-generating process is fully observed and free of hidden factors that jointly influence decisions and outcomes.
However, this assumption is often violated in many real-world settings. For example, consider human-robot interaction, routing and maneuvering of connected and automated vehicles in mixed traffic \cite{venkatesh2023connected,bang2025traffic}, medical records, and personalized decision systems \cite{venkatesh2025persuasion,dave2024airecommend}. 
In such settings, the offline dataset may omit important latent variables that affect both the decisions and the evolution of the system.
When actions depend on these hidden variables, the observed data may encode spurious correlations among state, action, next state, and reward. 
In causal inference, this phenomenon is referred to as \emph{confounding}. 
Hence, transition and reward models learned directly from observational data could potentially induce arbitrary correlations and bias into planning.
These difficulties are particularly severe in offline settings with partial observability, especially when the observed history is insufficient to recover, the latent state and the behavior policy may be correlated with unobserved factors that also affect rewards and transitions \cite{sutton2018reinforcement,levine2020offline}.

Recent progress on MBRL in confounded partially observable Markov decision processes (POMDPs) \cite{hong2024model} provides a bridge function-based framework for identifying reward and transition mechanisms from observational data. In particular, although these mechanisms are not directly identifiable from observational trajectories, they can be recovered through bridge functions satisfying suitable conditional moment restrictions (CMRs). Building on ideas from proximal causal inference and kernel methods, they propose a nonparametric two-stage procedure for estimating these bridge functions. This yields an estimator for off-policy evaluation and supports downstream policy optimization. More broadly, this perspective connects confounded sequential decision-making with the literature on proxy-based causal identification and two-stage learning under ill-posed inverse problems \cite{miao2018identifying,mastouri2021proximal,singh2019kernel,muandet2017kernel,song2009hilbert}.

Classical CMR approaches based on generalized method of moments \cite{hansen1982large}, sieve minimum distance and related conditional moment methods \cite{newey2003instrumental,ai2003efficient,chen2012estimation}, and kernel methods \cite{singh2019kernel} provide principled estimators for two-stage estimation problems. More recent efforts \cite{bennett2019deepgmm,singh2019kernel,hartford2017deepiv} employ deep neural networks to parameterize and estimate solutions to CMR problems. However, these estimators may still suffer from the propagation of first-stage errors into the target estimator through a naive plug-in step. From this viewpoint, the two-stage bridge-estimation procedure of \cite{hong2024model} can be naturally interpreted as a structured CMR problem.

Motivated by this connection, we revisit the two-stage bridge estimator of \cite{hong2024model} through the lens of cross-fitting. Building on recent work on double machine learning for CMRs \cite{shao2024learning}, we reformulate reward-emission and observation-transition bridge estimation as CMR problems with nuisance objects corresponding to the conditional mean embedding and the conditional density.
We then replace the single-split construction in \cite{hong2024model} with a $K$-fold cross-fitted procedure. For each fold, nuisance objects are estimated on the complementary sample, and the bridge criterion is evaluated on the held-out fold. The foldwise criteria are then aggregated to obtain the final estimator. This modification leaves the identification strategy unchanged and preserves the model-based structure of the original method. At the same time, it uses the available data more efficiently than a one-shot split and reduces the dependence of the estimator on any single partition of the sample.

Our aim in this paper is to develop a cross-fitted version of the bridge estimator of \cite{hong2024model} for confounded POMDPs and to clarify its connection with the general theory of CMR estimation. At the present stage, we view the cross-fitting construction itself as the main methodological contribution. The double machine learning literature suggests that orthogonality can further reduce the impact of first-stage regularization bias on the second stage. However, deriving a valid Neyman-orthogonal score for bridge estimation in confounded POMDPs is beyond the scope of the present work. We therefore leave orthogonality to future work. Even without this step, the proposed cross-fitted formulation provides a systematic refinement of the estimator in \cite{hong2024model} and leads to an oracle-comparator analysis that separates the contribution of nuisance estimation from that of empirical averaging.

The remainder of the paper is organized as follows. In Section~\ref{sec:problem_formulation}, we introduce the confounded POMDP model, state the bridge existence and completeness assumptions, and present the bridge-based representation of policy value. In Section~\ref{sec:bridge_learning}, we formulate bridge learning as a CMR problem, introduce the nuisance objects and the cross-fitted two-stage estimator, and establish the oracle-comparator decomposition.
In Section~\ref{sec:simulation}, we present a numerical example comparing a one-shot sample-splitting baseline with the proposed cross-fitted estimator to illustrate the improved sample efficiency of cross-fitting.
In Section~\ref{sec:conclusion}, we draw concluding remarks and outline several directions for future research.
Finally, in Appendix~\ref{app:stagewise}, we present the stage-wise error analysis and supporting bounds used in the proof.

\section{Problem Formulation}
\label{sec:problem_formulation}

We consider an episodic finite-horizon partially observable Markov decision process (POMDP) given by the tuple $\mathcal{M}=\bigl(\mathcal{X},\mathcal{Y},\mathcal{U},\mathcal{R},T,\nu_1,\{P_t\}_{t=1}^{T-1},\{E_t\}_{t=1}^{T},\{r_t\}_{t=1}^T\bigr)$, where $\mathcal{X}$, $\mathcal{Y}$, $\mathcal{U}$, and $\mathcal{R}$ denote the state, observation, action, and reward spaces, respectively.
The finite horizon is given by $T$.
The collections $\{P_t\}_{t=1}^{T-1}$ and $\{E_t\}_{t=1}^{T}$ denote the state transition and observation emission kernels, respectively.
The initial state distribution is denoted by $\nu_1\in\Delta(\mathcal{X})$, where $\Delta(\mathcal{X})$ is the set of all probability distributions over $\mathcal{X}$.
Throughout this paper, upper-case letters denote random variables, while lower-case letters denote their realizations.
At each time $t$, the random variables $X_t\in\mathcal{X}$, $Y_t\in\mathcal{Y}$, and $U_t\in\mathcal{U}$ denote the system state, observation, and control action, respectively.
The state of the system $X_t$ is unobserved, and we only have access to the observations $Y_t$ at each time $t$. 
At time $t$, the latent state evolves as $X_{t+1} \sim P_t(\,\cdot \mid X_t,U_t)$, with the observations generated as
$Y_t \sim E_t(\,\cdot \mid X_t)$.
The reward support is denoted by $\mathcal{R} = [-1,1]$ and the reward satisfies
$\mathbb{E}[R_t \mid X_t=x,U_t=u] = r_t(x,u),
\forall (x,u)\in\mathcal{X}\times\mathcal{U}$.

\begin{assumption}
At the beginning of each episode, before the first control action is selected, the decision maker receives an additional measurement $M \in \mathcal{Y}$ of the state of the system.
\end{assumption}

\begin{remark}
The variable $M$ models prior information that is often available in control applications before feedback decisions begin, such as an initial sensor snapshot, a calibration reading, or a baseline measurement. 
From a control viewpoint, $M$ is simply a pre-action measurement of the hidden initial condition. 
From an identification viewpoint, $M$ serves as an auxiliary baseline measurement whose role will be formalized later through the bridge-based assumptions.
\end{remark}

We define the observed history up to time $t$ by $H_t = (Y_1,U_1,\dots,Y_t,U_t)$, with the space of history given as the cartesian product $\mathcal{H}_t = \bigotimes_{j=1}^t (\mathcal{Y}\times\mathcal{U})$.
Since $M$ serves as an auxiliary observation available before the sequential interaction, we treat it separately from history $H_t$. 
A control policy determines the action at each time. 
Specifically, a history-dependent policy is a sequence of control laws given by $\boldsymbol{\pi}=\{{\pi}_t\}_{t=1}^T$ with $\pi_t:\mathcal{Y}\times\mathcal{Y}\times\mathcal{H}_{t-1}\to\Delta(\mathcal{U})$.
For any policy $\boldsymbol{\pi}$, we define the performance of the policy as the value of the policy, given by 
$V(\pi)=\mathbb{E}^{\pi}\left[\sum_{t=1}^T R_t\,\middle|\,X_1\sim \nu_1\right]$,
where $\mathbb{E}^{\pi}$ denotes that the expectation is with respect to the probability distribution induced by the policy $\boldsymbol{\pi}$ on the trajectory of the POMDP.

We consider that we have access to a dataset
$\mathcal{D}^{\boldsymbol{b}}= \{m^i,\{y^i_t,u^i_t,r^i_t\}_{t=1}^{T}\}_{i=1}^{N}$, generated by an unknown behavioral policy $\boldsymbol{\pi}^{\boldsymbol{b}}=\{\pi^{\boldsymbol{b}}_t\}_{t=1}^{T}$.
In the most general setting, the behavioral policy may depend on information beyond the observed quantities within the dataset, including the state of the system.  
We note that the state of the system $X_t$ is not recorded within the dataset, and all we can hope to learn about the system is the observation evolution and reward emission.

The broader objective is to evaluate the performance of any given policy  $\boldsymbol{\pi}^{\boldsymbol{e}}$.
Existing bridge-based identification results \cite{kuangbreaking,bennett2024proximal,hong2024model} show that policy evaluation reduces to learning stage-wise bridge functions. 
A bridge function is an observable object whose conditional expectation recovers a target conditional law involving unobserved variables. 
In particular, we consider reward and observation transition bridge functions  \cite{hong2024model} that recover the reward-emission and observation-transition laws used in bridge-based policy evaluation. 
Consequently, when such bridge functions exist, the value of a given policy $\boldsymbol{\pi}^{\boldsymbol{e}}$ admits a representation in terms of observed variables alone. 
Motivated by this connection, we first state the assumptions underlying this bridge-based reformulation and then present the resulting bridge-based representation of policy value. 
We turn next to bridge learning and develop a cross-fitted estimator for the bridge functions.

\begin{assumption}[Bridge existence]
\label{bridge_existence}
We assume that there exist reward-emission bridge functions $\boldsymbol{b}^R=\bigl\{
b^R_t:\mathcal{U}\times\mathcal{Y}\times\mathcal{R}\times\mathcal{Y}\to\mathbb{R}
\bigr\}_{t=1}^T$ 
such that, for each $t=1,\dots,T$ and for all $(r_t,y_t) \in \mathcal{R} \times \mathcal{Y}$  
\begin{align}
\nonumber \mathbb{E}
\bigl[
b^R_t(U_t,Y_t,r_t,y_t)&
\mid H_{t-1},U_t,M
\bigr]\\
&=
p(r_t,y_t\mid H_{t-1},U_t,M),\label{as:r_bridge_condition} 
\end{align}
and there exist observation-transition bridge functions $\boldsymbol{b}^D
=\bigl\{
b^D_t:\mathcal{U}\times\mathcal{Y}\times\mathcal{Y}\times\mathcal{Y}\to\mathbb{R}
\bigr\}_{t=1}^{T-1}$ such that, for each $t=1,\dots,T-1$ and for all $(y_{t+1},y_t) \in \mathcal{Y} \times \mathcal{Y}$
\begin{align}
\nonumber \mathbb{E}
\bigl[
b^D_t(U_t,Y_t,y_{t+1},y_t)&
\mid H_{t-1},U_t,M
\bigr]\\
&=
p(y_{t+1},y_t\mid H_{t-1},U_t,M).\label{as:d_bridge_condition}
\end{align}
\end{assumption}
\begin{remark}
Assumption~\ref{bridge_existence} is an existence condition for the bridge functions and is standard in bridge-based identification \cite{miao2018confounding}. In particular, bridge existence follows under suitable regularity and completeness conditions on the relevant conditional expectation operators \cite{bennett2024proximal,hong2024model}.
\end{remark}

\begin{assumption}[Bridge completeness]
\label{bridge_completeness}
We assume that for each $t=1,\dots,T$, if a measurable function $g_t:\mathcal{X}\times\mathcal{U}\to\mathbb{R}$ satisfies
\begin{align}
\mathbb{E}\bigl[g_t(X_t,U_t)\mid U_t,H_{t-1},M\bigr]=0
\quad \text{a.s.},
\end{align}
then $g_t(X_t,U_t)=0$ holds almost surely.
\end{assumption}

\begin{remark}
Assumption~\ref{bridge_completeness} is an injectivity condition on the conditional expectation map $g_t \mapsto \mathbb{E}[g_t(X_t,U_t)\mid U_t,H_{t-1},M]$.
It ensures that the latent state-action quantity appearing in the bridge-based policy-value formula is uniquely determined by observable conditional moments.
This is the notion of uniqueness needed for the identification of policy value.
However, the bridge functions themselves need not be unique.
Multiple bridge functions may induce the same latent conditional expectation and hence the same policy value.
\end{remark}

\noindent To connect bridge learning to policy evaluation, we state the bridge-based identification step in the notation of our work, following \cite{hong2024model}. For each $t=1,\dots,T$, define
\begin{align}
\nonumber f_t(r_t,h_t,m)
=
\int_{\mathcal{Y}^{t-1}}
&b_t^R(u_t,\tilde y_{t-1},r_t,y_t) \\
&\cdot
\prod_{j=1}^{t-1}
b_j^D(u_j,\tilde y_{j-1},\tilde y_j,y_j)\,
d\tilde y_{1:t-1},
\label{eq:ft_bridge_def}
\end{align}
where $h_t=(y_1,u_1,\dots,y_t,u_t)$, is the realization of history at time $t$ and $\tilde y_0 = m$.
The empty product is interpreted as  for $t=1$, $
f_1(r_1,h_1,m)=b_1^R(u_1,m,r_1,y_1)$.

The following proposition is the analogue, in the present notation, of the main identification step in \cite{hong2024model}.

\begin{proposition}
\label{prop:policy_val_identification}
Suppose Assumptions~\ref{bridge_existence}--\ref{bridge_completeness} hold. Then,  for each $t=1,\dots,T$, the stage-wise reward law under a history-dependent policy $\boldsymbol{\pi}$ is identified by
\begin{align}
\nonumber p^\pi(r_t)
=
\int_{\mathcal Y}\int_{\mathcal H_t}
\left(
\prod_{j=1}^t \pi_j(u_j\mid y_j,m,h_{j-1})
\right)\\
&\hspace{-100pt}
\cdot f_t(r_t,h_t,m)\,
d h_t\, p(m)\, d m.
\label{eq:reward_marginal}
\end{align}
Therefore, the policy value is identified by
\begin{align}
\nonumber V(\pi)
=
\sum_{t=1}^T
\int_{\mathcal R}
\int_{\mathcal Y}
\int_{\mathcal H_t}
r_t
\left(
\prod_{j=1}^t \pi_j(u_j\mid y_j,m,h_{j-1})
\right)\\
& \hspace{-100pt}
f_t(r_t,h_t,m)\,
d h_t\, p(m)\, d m\, d r_t,
\label{eq:value_identification}
\end{align}
where, $p(m)$ denotes the distribution of the additional measurement $M$.
\end{proposition}

Hence, policy evaluation reduces to identifying the collection of stage-wise reward laws $\{p^\pi(r_t)\}_{t=1}^T$. By Proposition~\ref{prop:policy_val_identification}, this further reduces to estimating the reward-emission bridges $\{b_t^R\}_{t=1}^T$, the observation-transition bridges $\{b_t^D\}_{t=1}^{T-1}$, and the distribution of the additional measurement. The latter can be estimated empirically, so the main statistical challenge lies in learning the bridge functions.
This motivates the following problem.

\begin{problem}
\label{prob:bridge_learning}
Given a dataset $\mathcal{D}^b$,
generated under an unknown behavioral policy $\boldsymbol{\pi}^{\boldsymbol{b}}$, estimate the stage-wise reward-emission bridge functions $\{b_t^R\}_{t=1}^T$ and observation-transition bridge functions $\{b_t^D\}_{t=1}^{T-1}$ from observed data alone.
\end{problem}

The goal is to construct estimators $\{\hat b_t^R\}_{t=1}^T$ and $\{\hat b_t^D\}_{t=1}^{T-1}$ that approximately satisfy the bridge relations in Assumption~\ref{bridge_existence}.
Then, the approximate bridge functions can be used in the identified representation of the stage-wise reward laws and, consequently, of the policy value.
In the next section, we develop a cross-fitted two-stage bridge-learning procedure for this task.


\section{Bridge Learning}
\label{sec:bridge_learning}
In this section, we formulate bridge estimation as a generic CMR problem.
Motivated by the bridge-learning framework of \cite{hong2024model} and the nuisance-target separation used in double machine learning \cite{chernozhukov2018double}, we develop a cross-fitted two-stage estimation procedure over a prescribed bridge class $\mathcal{B}$. 
Next, we introduce a risk functional under the true data-generating distribution and its cross-fitted empirical counterpart on the class $\mathcal{B}$.
These functionals allow us to compare the empirical cross-fitted estimator with the best-in-class bridge in $\mathcal{B}$.

Since both bridge equations have the same conditional moment structure, we introduce a generic bridge-learning problem that covers both cases.
At each time $t$, we introduce the generic variables $W_t = (U_t,Y_t)$,
$Z_t^R = (R_t,Y_t)$,
$Z_t^D = (Y_{t+1},Y_t)$,
$C_t = (H_{t-1},U_t,M)$.
The corresponding spaces are given by $\mathcal{W}_t= \mathcal{U} \times \mathcal{Y}$, $\mathcal{Z}_t^R= \mathcal{R} \times \mathcal{Y}$, $\mathcal{Z}_t^D= \mathcal{Y} \times \mathcal{Y}$, and $\mathcal{C}_t= \mathcal{H}_{t-1} \times \mathcal{U}\times\mathcal{Y}$.
Accordingly, the corresponding bridge equations are
\begin{align}
\mathbb{E}\bigl[b^R_{t}(W_t,z_t^R)\mid C_t\bigr]
&=
p(z_t^R\mid C_t),
\qquad z_t^R\in \mathcal{Z}_t^R,\\
\mathbb{E}\bigl[b^D_{t}(W_t,z_t^D)\mid C_t\bigr]
&=
p(z_t^D\mid C_t),
\qquad z_t^D\in\mathcal{Z}_t^D.
\end{align}
To simplify the notation, we suppress the time index and write $W= W_t$, $Z\in \{ Z^R_t,Z^D_t\}$, and $C= C_t$.
Here, $Z$ denotes either $(R_t,Y_t)$ or $(Y_{t+1},Y_t)$, depending on the bridge under consideration.
The generic bridge-learning problem then is to estimate a function $b \in \mathcal{B}$ that satisfies the generic CMR
\begin{align}
\mathbb{E}[b(W,z)\mid C] = p(z\mid C),
\qquad z\in\mathcal{Z}.
\label{eq:generic_cmr}
\end{align}

More specifically, we recall that we have access to the dataset $\mathcal{D}^b$. Hence, for the generic CMR, we can construct a dataset $\mathcal{D}$ based on samples of the generic random variables $(W,Z,C)$. Essentially, this allows us to consider solving the CMR in \eqref{eq:generic_cmr}, given a dataset of i.i.d samples $\mathcal{D}=\{(w_i,c_i,z_i)\}_{i=1}^N$.
The bridge-learning problem is then to estimate a function $b$ that approximately satisfies the generic CMR in \eqref{eq:generic_cmr}.

\subsection{Bridge class and kernel representation}

We adopt a kernel-based formulation of the bridge-learning problem, following \cite{song2009hilbert,hong2024model}. 
We consider a bridge class $\mathcal{B}\subseteq \mathcal{H}_W\otimes\mathcal{H}_Z$, where $\mathcal{H}_W$ and $\mathcal{H}_Z$ are reproducing kernel Hilbert spaces (RKHS). 
The bridge equation in \eqref{eq:generic_cmr} involves a conditional expectation operator applied to the bridge function $b\in\mathcal{B}$. 
We use an RKHS representation to express this conditional expectation in a linear inner-product form through conditional mean embeddings.
To obtain this representation, we map the variables $W$ and $Z$ to their corresponding RKHSs. 
Let $\varphi:\mathcal{W}\to\mathcal{H}_W$ and $\phi:\mathcal{Z}\to\mathcal{H}_Z$ denote the canonical feature maps.
The conditional mean embedding $\mu_{W\mid C}$, which summarizes the conditional law of $W$ given $C$ in the RKHS $\mathcal{H}_W$, is defined by
\begin{align}
\mu_{W\mid C}(c)
=
\mathbb{E}\bigl[\varphi(W)\mid C=c\bigr],
\qquad c\in\mathcal{C}.
\end{align}
Then, for any $b\in\mathcal{B}$, and any $c\in\mathcal{C}$, the conditional expectation operation in \eqref{eq:generic_cmr} can be written as
\begin{align}
\mathbb{E}[b(W,z)\mid C=c]
=
\left\langle
\mu_{W\mid C}(c)\otimes\phi(z),\, b
\right\rangle, \;\; z\in\mathcal{Z},
\label{eq:RKHS_bridge_inner_product}
\end{align}
where $\langle\cdot,\cdot\rangle$ denotes the inner product on $\mathcal{H}_W\otimes\mathcal{H}_Z$.
We note that, for each $c\in\mathcal{C}$ and $z\in\mathcal{Z}$, the tensor product $\mu_{W\mid C}(c)\otimes\phi(z)$ is an element of $\mathcal{H}_W\otimes\mathcal{H}_Z$. 
Since $b\in\mathcal{B}\subseteq\mathcal{H}_W\otimes\mathcal{H}_Z$, the inner product in \eqref{eq:RKHS_bridge_inner_product} is well defined.

The representation in \eqref{eq:RKHS_bridge_inner_product} describes the conditional moment under the true law of $(W,C,Z)$. 
Since in practice we only observe the sample $\mathcal{D}$, these quantities are not directly available. 
Therefore, we distinguish between population objects, defined under the true law of $(W,C,Z)$, and empirical objects, constructed from the sample $\mathcal{D}$.
To assess how well a bridge function $b\in\mathcal{B}$ satisfies the CMR in \eqref{eq:generic_cmr}, we introduce the following population risk functional.

\begin{definition}
Let $\nu$ be a fixed reference probability measure on $\mathcal{Z}$. For any $b\in\mathcal{B}$, we define
\begin{align}
L(b)
&=
\int_{\mathcal{Z}}
\mathbb{E}
\left[
\left(
\mathbb{E}[b(W,z)\mid C]-p(z\mid C)
\right)^2
\right]
\,d\nu(z),
\\
&=
\int_{\mathcal{Z}}
\mathbb{E}
\left[
\left(
\left\langle
\mu_{W\mid C}(C)\otimes\phi(z),\, b
\right\rangle
-
p(z\mid C)
\right)^2
\right]
\,d\nu(z).
\label{cmr_pop_true_nuisance}
\end{align}
\end{definition}

The risk $L(b)$ aggregates the squared residual in \eqref{eq:generic_cmr} over $z\in\mathcal{Z}$ and provides a scalar criterion for comparing bridge functions in $\mathcal{B}$.
Since \eqref{eq:generic_cmr} is a functional condition in $z$, the measure $\nu$ determines how the CMR residual is aggregated over $\mathcal{Z}$.
 
In the empirical bridge-learning procedure, we estimate $b$ over the prescribed class $\mathcal{B}$ and include regularization to stabilize estimation.  
Accordingly, we consider the corresponding penalized population objective given by
\begin{align}
L_{\lambda}(b)
=
L(b)+\lambda\|b\|_{\mathcal{H}_W\otimes\mathcal{H}_Z}^{2}.
\label{eq:penalized_population_risk}
\end{align}

Later, we compare the cross-fitted empirical bridge estimator with the best penalized bridge in $\mathcal{B}$ under the true law of $(W,C,Z)$. This motivates the following definition of an oracle comparator.

\begin{definition}[Oracle comparator]\label{def2}
The oracle comparator is any minimizer of the penalized population objective over the bridge class $\mathcal{B}$:
\begin{align}
b_{\lambda}^{\dagger}
\in
\arg\min_{b\in\mathcal{B}} L_{\lambda}(b).
\label{eq:oracle_comparator}
\end{align}
\end{definition}

\begin{remark}
We do not assume that the working bridge class $\mathcal{B}$ contains a bridge satisfying \eqref{eq:generic_cmr}. Accordingly, $b_{\lambda}^{\dagger}$ is defined only as the best penalized approximation over $\mathcal{B}$.
\end{remark}

Next, we turn to the cross-fitted two-stage bridge-learning procedure. 
Stage I estimates the nuisance objects that enter the bridge criterion, while Stage II uses these estimates as plug-ins to construct the bridge estimator over the class $\mathcal{B}$. 
First, we introduce the nuisance objects and then define the corresponding cross-fitted empirical risk.

\subsection{Nuisance objects}

The bridge criterion in \eqref{cmr_pop_true_nuisance} depends on the conditional mean embedding of $W$ given $C$ and the conditional density of $Z$ given $C$.
These quantities are not the primary inferential targets, but they must be estimated to construct the bridge estimator.
Following the terminology of double machine learning, we refer to such intermediate quantities as nuisance objects.

Earlier, $\mu_{W\mid C}$ and $p(z\mid C)$ were used to denote the generic population conditional mean embedding and conditional density.
For notational clarity, from this point onward, we attach the superscript $0$ to indicate the corresponding population objects as
\begin{align}
\eta^{0}
=
\bigl(
\mu_{W\mid C}^{0},\,
p^{0}(\cdot\mid C)
\bigr),
\end{align}
where $\eta^{0}$ denotes the population nuisance collection..
Based on $\eta^{0}$, for any $(c,z)\in\mathcal{C}\times\mathcal{Z}$, we define the corresponding population residual by
\begin{align}
r(b,c,z;\eta^{0})
=
\left\langle
\mu_{W\mid C}^{0}(c)\otimes \phi(z),\, b
\right\rangle
-
p^{0}(z\mid c).
\end{align}
This residual measures the pointwise discrepancy in the CMR for a given bridge $b$.

\subsection{Cross-fitted nuisance and bridge estimation}

Let $K\in\mathbb{N}$ be fixed and assume that $K$ divides $N$, the number of observations in the sample $\mathcal{D}$. Let $\{I_k\}_{k=1}^K$ be a partition of $[N]=\{1,\dots,N\}$ into $K$ disjoint folds equal size $|I_k|=n=N/K$. 
For each $k=1,\dots,K$, let $
I_k^c = [N]\setminus I_k$, $\mathcal{D}_{I_k} = \{(w_i,c_i,z_i):i\in I_k\}$, and $
\mathcal{D}_{I_k^c} = \{(w_i,c_i,z_i):i\in I_k^c\}$.
We refer to $\mathcal{D}_{I_k}$ as the held-out fold and to $\mathcal{D}_{I_k^c}$ as the auxiliary sample.
Next, we describe the two-stage estimation procedure for the bridge function.

\paragraph{Stage I - nuisance estimation}
Stage I requires an estimator of the conditional mean embedding of $W$ given $C$ and an estimator of the conditional density of $Z$ given $C$.
For each fold $k$, we estimate the nuisance objects using only the auxiliary sample $\mathcal{D}_{I_k^c}$.
We denote the resulting estimators by
\begin{align}
\hat{\eta}^{k}
=
\bigl(
\hat{\mu}_{W\mid C}^{\,k},\,
\hat{p}^{\,k}(\cdot\mid C)
\bigr).
\end{align}
The superscript $k$ labels the held-out fold used in Stage II, while the estimators themselves are trained on the auxiliary sample $\mathcal{D}_{I_k^c}$.
For each fold $k$, we define the fold-specific plug-in residual for $(c,z)\in\mathcal{C}\times\mathcal{Z}$ by
\begin{align}
r(b,c,z;\hat{\eta}^{k})
=
\left\langle
\hat{\mu}_{W\mid C}^{k}(c)\otimes \phi(z),\, b
\right\rangle
-
\hat{p}^{k}(z\mid c).
\end{align}

We adapt the kernel-based conditional mean embedding framework of \cite{song2013kernel} to estimate $\hat{\mu}^{\,k}_{W\mid C}$ and use a conditional likelihood-based learner to estimate $\hat{p}^{\,k}(\cdot\mid C)$.
Specifically, let $\psi:\mathcal{C}\to\mathcal{H}_C$ denote a feature map on $\mathcal{C}$, and let $\mathcal{H}_{C\to W}$ be an operator-valued RKHS of maps from $\mathcal{H}_C$ to $\mathcal{H}_W$.
A convenient estimator of the conditional mean embedding is the regularized operator regression
\begin{align}
\nonumber
\hat{C}_{W\mid C}^{k}
\in
\argmin_{C\in\mathcal{H}_{C\to W}}
\frac{1}{|I_k^c|}
\sum_{i\in I_k^c}
\left\|
\varphi(w_i)-C\psi(c_i)
\right\|_{\mathcal{H}_W}^{2}\\
+
\lambda^k_{1}\|C\|_{\mathcal{H}_{C\to W}}^{2},
\end{align}
and the corresponding nuisance estimate is given by $
\hat{\mu}_{W\mid C}^{k}(c)
=
\hat{C}_{W\mid C}^{k}\psi(c)$, for $c\in\mathcal{C}$.

For the conditional density, we use a conditional likelihood-based estimator of the form
\begin{align}
\hat{p}^{k}
\in
\argmax_{p\in\mathcal{P}}
\frac{1}{|I_k^c|}
\sum_{i\in I_k^c}
\log p(z_i\mid c_i),
\end{align}
for a prescribed model class $\mathcal{P}$.
\vspace{5pt}

\paragraph{Stage II - bridge estimation}
Given the nuisance estimates $\hat{\eta}^{k}$, we evaluate candidate bridge functions on the held-out fold $I_k$.
To define the empirical analogue of the population risk in \eqref{cmr_pop_true_nuisance}, the integral over $\mathcal{Z}$ is approximated by Monte Carlo simulation.
For each fold $k$, let $\{\tilde z_{m}^{\,k}\}_{m=1}^{M_k}$ be i.i.d. auxiliary draws from $\nu$, independent of the held-out data $\mathcal{D}_{I_k}$, conditional on $\mathcal{D}_{I_k^c}$. 
The foldwise empirical loss is then defined by
\begin{align}
\hat{L}_k(b)
&=
\frac{1}{|I_k|}
\sum_{i\in I_k}
\frac{1}{M_k}
\sum_{m=1}^{M_k}
\left(
r(b,c_i,\tilde z_m^{\,k};\hat{\eta}^{k})
\right)^2.
\label{eq:foldwise_empirical_loss}
\end{align}

Thus, the held-out sample $\mathcal{D}_{I_k}$ averages over the distribution of $C$, while the auxiliary draws $\{\tilde z_m^{\,k}\}_{m=1}^{M_k}$ approximate integration over $\mathcal{Z}$ with respect to $\nu$.

\noindent The corresponding cross-fitted empirical risk for any $b \in \mathcal{B}$ is 
\begin{align}
\hat{L}^{\cf}(b)
=
\frac{1}{K}\sum_{k=1}^K \hat{L}_k(b)
+
\lambda \|b\|_{\mathcal{H}_W\otimes\mathcal{H}_Z}^{2},
\label{eq:cf_empirical_objective}
\end{align}
where $\lambda$ is the common stage II regularization parameter. 
The cross-fitted bridge estimator is then defined by
\begin{align}
\hat{b}^{\cf}
\in
\argmin_{b\in\mathcal{B}}
\hat{L}^{\cf}(b).
\label{eq:cf_bridge_estimator}
\end{align}

Before comparing the cross-fitted bridge estimator and the oracle comparator, we introduce an intermediate risk.
The intermediate risk separates the effect of nuisance estimation from that of empirical averaging.
The foldwise intermediate risk is given by
\begin{align}
\tilde{L}_k(b)
&=
\int_{\mathcal{Z}}
\mathbb{E}
\left[
\left(
r(b,C,z;\hat{\eta}^{k})
\right)^2
\,\middle|\,
\mathcal{D}_{I_k^c}
\right]
\,d\nu(z),
\label{eq:fold_plugin_population_loss}
\end{align}
and the corresponding cross-fitted intermediate risk is
\begin{align}
\tilde{L}^{\cf}(b)
&=
\frac{1}{K}\sum_{k=1}^K \tilde{L}_k(b)
+
\lambda \|b\|_{\mathcal{H}_W\otimes\mathcal{H}_Z}^{2}.
\label{eq:cf_plugin_population_loss}
\end{align}

This allows us to decompose the oracle comparison error into a Stage I term induced by nuisance estimation and a Stage II term induced by empirical averaging.
The next theorem formalizes this decomposition.

\begin{theorem}[Oracle theorem]
\label{thm:oracle theorem}
Let $b_{\lambda}^{\dagger}$ denote the oracle comparator from Definition~\ref{def2}. Then
\begin{align}
L_{\lambda}(\hat{b}^{\cf})&-L_{\lambda}(b_{\lambda}^{\dagger})
\le
2\sup_{b\in\mathcal{B}}
\left|
L_{\lambda}(b)-\hat{L}^{\cf}(b)
\right|
\label{oracle_comparator_bound}
\\
&\le
2\sup_{b\in\mathcal{B}}
\left|
L_{\lambda}(b)-\tilde{L}^{\cf}(b)
\right|
+
2\sup_{b\in\mathcal{B}}
\left|
\tilde{L}^{\cf}(b)-\hat{L}^{\cf}(b)
\right|.
\label{eq:stage_wise_split}
\end{align}
\end{theorem}

\begin{proof}
By adding and subtracting $\hat{L}^{\cf}(\hat{b}^{\cf})$ and $\hat{L}^{\cf}(b_{\lambda}^{\dagger})$, we obtain
\begin{align}
\nonumber L_{\lambda}(\hat{b}^{\cf})-L_{\lambda}(b_{\lambda}^{\dagger})
&=
L_{\lambda}(\hat{b}^{\cf})-\hat{L}^{\cf}(\hat{b}^{\cf})
+\\
&\hat{L}^{\cf}(\hat{b}^{\cf})
-
\hat{L}^{\cf}(b_{\lambda}^{\dagger})
+
\hat{L}^{\cf}(b_{\lambda}^{\dagger})-L_{\lambda}(b_{\lambda}^{\dagger}).
\end{align}
By definition, $\hat{b}^{\cf}$ is a minimizer of $\hat{L}^{\cf}$ over $\mathcal{B}$, we have
\begin{align}
\hat{L}^{\cf}(\hat{b}^{\cf})
\le
\hat{L}^{\cf}(b_{\lambda}^{\dagger}).
\end{align}
Hence,
\begin{align}
L_{\lambda}(\hat{b}^{\cf})-L_{\lambda}(b_{\lambda}^{\dagger})
&\le
L_{\lambda}(\hat{b}^{\cf})-\hat{L}^{\cf}(\hat{b}^{\cf})
+
\hat{L}^{\cf}(b_{\lambda}^{\dagger})-L_{\lambda}(b_{\lambda}^{\dagger})
\\
&\le
2\sup_{b\in\mathcal{B}}
\left|
L_{\lambda}(b)-\hat{L}^{\cf}(b)
\right|.
\end{align}
The second bound follows from the triangle inequality:
\begin{align}
\left|
L_{\lambda}(b)-\hat{L}^{\cf}(b)
\right|
\le
\left|
L_{\lambda}(b)-\tilde{L}^{\cf}(b)
\right|
+
\left|
\tilde{L}^{\cf}(b)-\hat{L}^{\cf}(b)
\right|.
\end{align}
Taking the supremum over $b\in\mathcal{B}$ gives \eqref{eq:stage_wise_split}.
\end{proof}
The first term in \eqref{eq:stage_wise_split} is the nuisance-induced error from Stage I.
The second term is the empirical-process error from Stage 2. To provide bounds on the errors from each stage, we state our assumptions and provide a detailed analysis in Appendix~\ref{app:stagewise}.

\section{Simulation}
\label{sec:simulation}

We consider a $2$-step confounded POMDP and focus on the first-stage observation-transition bridge. The system is a low-dimensional nonlinear Gaussian POMDP with binary actions, and the latent state contains a hidden binary contextual factor. The observations consist only of the visible state coordinate and proxy measurements.

We construct the POMDP so that both the logged action and the next observed state depend on the hidden contextual component. Consequently, the observational transition is generally biased relative to the intervention target relevant for planning. The observation-transition bridge is used to recover this deconfounded observation evolution from observable data.
The bridge relation in \eqref{as:d_bridge_condition} identifies the full deconfounded observable transition law. For visualization, we consider a low-dimensional functional of this law, namely the conditional mean of the visible state coordinate contained in $y_{t+1}$. This mean therefore provides a natural and tractable summary of CMR satisfaction. Thus, the reported mean-squared error should be understood as a summary of observation-transition recovery, rather than as a metric for recovery of the full conditional distribution.

We compare a one-shot $50$--$50$ sample-splitting baseline with the proposed $K$-fold cross-fitted estimator. Figure~\ref{fig:transition_mse} reports the mean-squared error for recovery of the deconfounded first-stage observation-transition target as a function of sample size. Cross-fitting improves recovery across the full sample-size range considered: the error decreases from $0.1749$ to $0.0956$ at $N=120$, from $0.1627$ to $0.0797$ at $N=240$, from $0.1032$ to $0.0517$ at $N=480$, from $0.0907$ to $0.0339$ at $N=960$, and from $0.0771$ to $0.0248$ at $N=1920$. These results suggest that both methods improve with larger samples, while the cross-fitted estimator is consistently more sample-efficient than one-shot sample splitting in this setting.

\begin{figure}[t]
    \centering
    \includegraphics[width=\columnwidth]{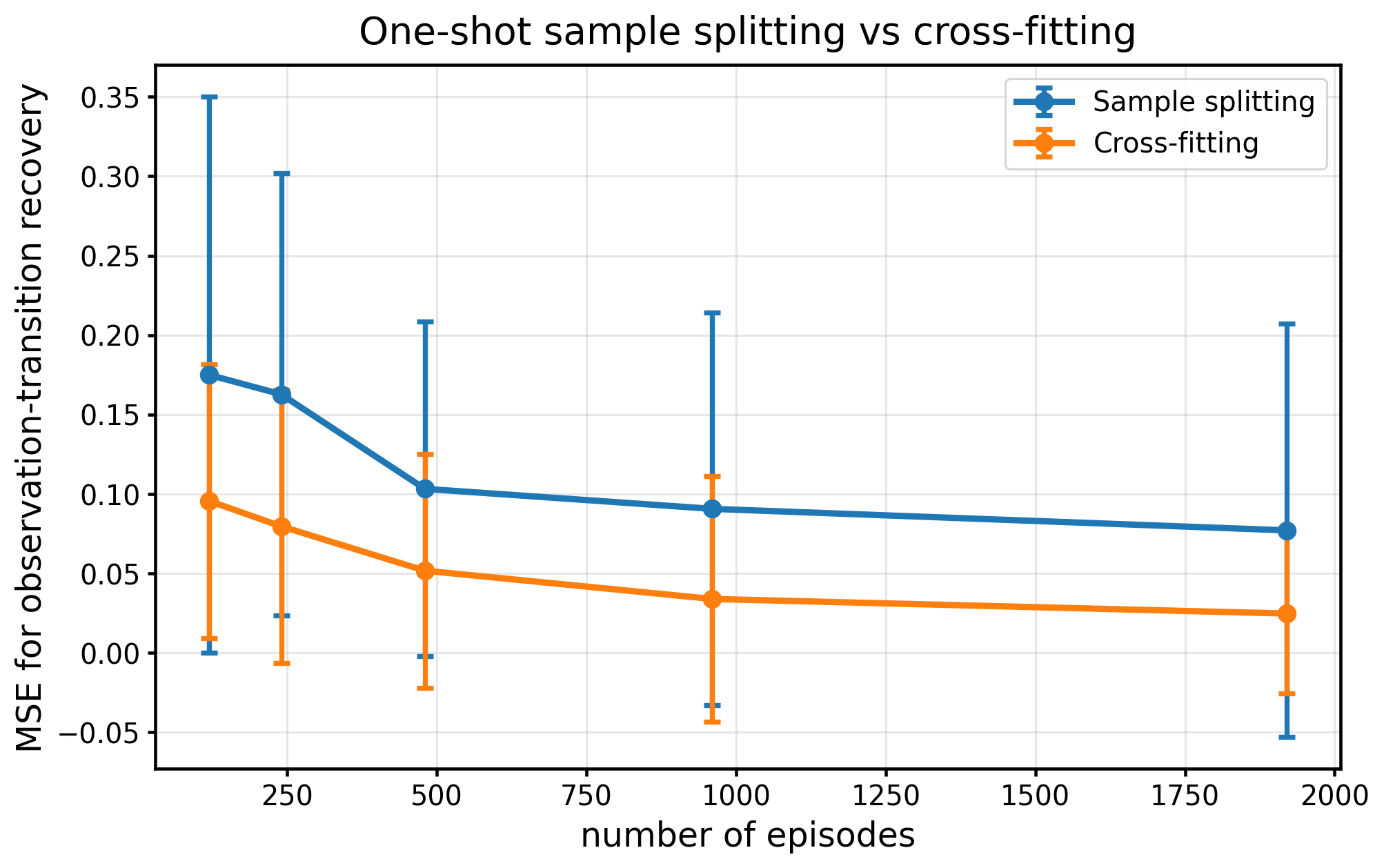}
    \caption{Comparison of observation-transition recovery}
    \label{fig:transition_mse}
\end{figure}

\section{Conclusion}
\label{sec:conclusion}

In this paper, we provided a principled statistical refinement of existing bridge-learning methods for model-based offline reinforcement learning under hidden confounding.
We formulated bridge estimation as a conditional moment restriction problem with nuisance components given by a conditional mean embedding and a conditional density. This leads to a $K$-fold cross-fitted two-stage estimator and enables a sharper statistical analysis than a single-sample split. Our main result establishes an oracle comparator bound and decomposes the excess risk into a first-stage nuisance-estimation term and a second-stage empirical-process term. We provide a principled extension of existing bridge-learning methods for model-based offline reinforcement learning under hidden confounding.

This paper focused on the bridge-learning problem and its statistical analysis. In particular, we studied cross-fitted bridge estimation and its oracle-style guarantees, while leaving the development of Neyman-orthogonal or doubly robust bridge estimators, as well as end-to-end guarantees for policy optimization, to future work. Future work should strengthen the oracle-style analysis into explicit finite-sample bounds under concrete assumptions on the nuisance learners and the bridge class. It should also translate the bridge-estimation error into downstream policy-value error bounds. A further direction is to develop orthogonalized bridge-learning objectives that are less sensitive to first-stage regularization bias and better suited to flexible nuisance estimation.

\bibliographystyle{ieeetr}

\bibliography{References,Latest_IDS}
\vspace{-10pt}
\appendices
\section{Stage-wise Error Analysis}
\label{app:stagewise}

The decomposition in Theorem~\ref{thm:oracle theorem} holds in general for any bridge class $\mathcal{B}$.
To convert it into high-level bounds, next, we impose regularity assumptions and analyze the Stage I and Stage II terms separately.
Throughout this appendix, $\|\cdot\|_{\mathcal{H}_W\otimes\mathcal{H}_Z}$ denotes the RKHS norm on $\mathcal{H}_W\otimes\mathcal{H}_Z$.

\begin{assumption}[Bounded feature map]
\label{bounded_feature_map}
There exists a constant $K_Z<\infty$ such that, $\sup_{z\in\mathcal{Z}}
\|\phi(z)\|_{\mathcal{H}_Z}
\le
K_Z.
$
\end{assumption}

\begin{assumption}[Bounded bridge class]
\label{bounded_bridge_class}
There exists a constant $B<\infty$ such that, $\sup_{b\in\mathcal{B}}
\|b\|_{\mathcal{H}_W\otimes\mathcal{H}_Z}
\le
B.
$
\end{assumption}

\subsection{Stage I error}

We begin with the first term in \eqref{eq:stage_wise_split}.
By the definition of $\tilde{L}^{\cf}$ and the triangle inequality,
\begin{align}
\sup_{b\in\mathcal{B}}
\left|
L_{\lambda}(b)-\tilde{L}^{\cf}(b)
\right|
&=
\sup_{b\in\mathcal{B}}
\left|
\frac{1}{K}\sum_{k=1}^K
\bigl(
L(b)-\tilde{L}_k(b)
\bigr)
\right| 
\\
&\le
\frac{1}{K}\sum_{k=1}^K
\sup_{b\in\mathcal{B}}
\left|
L(b)-\tilde{L}_k(b)
\right|.
\label{eq:stage1_average_bound}
\end{align}
It therefore suffices to control the foldwise absolute difference
$|L(b)-\tilde{L}_k(b)|$.

\begin{lemma}
\label{lem:foldwise_stage1_identity}
For each fold $k$ and each $b\in\mathcal{B}$,
\begin{align}
\nonumber L(b)&-\tilde{L}_k(b)
=\\
&\int_{\mathcal{Z}}
\mathbb{E}
\Bigl[
r(b,C,z;\eta^{0})^2
-
r(b,C,z;\hat{\eta}^{k})^2
\,\Big|\,
\mathcal{D}_{I_k^c}
\Bigr]
\,d\nu(z).\label{eq:true_nuisance_aux_indep}
\end{align}
\end{lemma}

\begin{proof}
By the definition of $\tilde{L}_k(b)$,
\begin{align}
\nonumber &\int_{\mathcal{Z}}
\mathbb{E}
\Bigl[
r(b,C,z;\eta^{0})^2
-
r(b,C,z;\hat{\eta}^{k})^2
\,\Big|\,
\mathcal{D}_{I_k^c}
\Bigr]
\,d\nu(z)
\\
&=
\int_{\mathcal{Z}}
\mathbb{E}
\Bigl[
r(b,C,z;\eta^{0})^2
\,\Big|\,
\mathcal{D}_{I_k^c}
\Bigr]
\,d\nu(z)
-
\tilde{L}_k(b).
\end{align}
Since $\eta^{0}$ is nonrandom with respect to the sigma algebra $\sigma(\mathcal{D}_{I_k^c})$, we have 
\begin{align}
\mathbb{E}
\Bigl[
r(b,C,z;\eta^{0})^2
\,\Big|\,
\mathcal{D}_{I_k^c}
\Bigr]
=
\mathbb{E}
\bigl[
r(b,C,z;\eta^{0})^2
\bigr].
\end{align}
Substituting this into the previous equation gives us \eqref{eq:true_nuisance_aux_indep}.
\end{proof}

We apply Lemma~\ref{lem:foldwise_stage1_identity} and consider the identity
$|\mathbb{E}[X\mid\mathcal{G}]|
\le
\mathbb{E}[|X|\mid\mathcal{G}]$ to get
\begin{align}
\nonumber &\sup_{b\in\mathcal{B}}
\left|
L(b)-\tilde{L}_k(b)
\right|
\\
&\le
\sup_{b\in\mathcal{B}}
\int_{\mathcal{Z}}
\mathbb{E}
\Bigl[
\bigl|
r(b,C,z;\eta^{0})^2
-
r(b,C,z;\hat{\eta}^{k})^2
\bigr|
\,\Big|\,
\mathcal{D}_{I_k^c}
\Bigr]
\,d\nu(z).
\label{eq:stage1_abs_diff}
\end{align}

\noindent Next, we use the identity
\begin{align}
|a^2-d^2|
=
|a-d||a+d|
\le
|a-d|^2 + 2|d||a-d|.    
\end{align}
We apply this pointwise with
$a=r(b,C,z;\hat{\eta}^{k})$
and
$d=r(b,C,z;\eta^{0})$.
Substituting into \eqref{eq:stage1_abs_diff} gives
\begin{align}
\sup_{b\in\mathcal{B}}
\left|
L(b)-\tilde{L}_k(b)
\right|
\le
\sup_{b\in\mathcal{B}}
\Delta_k(b)
+
2\sup_{b\in\mathcal{B}}
\Gamma_k(b),
\label{eq:stage1_split}
\end{align}
where, 
\begin{align}
&\Delta_k(b)
:=
\int_{\mathcal{Z}}
\mathbb{E}
\Bigl[
\bigl|
a-d\bigr|^2
\,\Big|\,
\mathcal{D}_{I_k^c}
\Bigr]
\,d\nu(z),
\\
&\Gamma_k(b)
:=\int_{\mathcal{Z}}
\mathbb{E}
\Bigl[
|d|
\,
\bigl|a-d
\bigr|
\,\Big|\,
\mathcal{D}_{I_k^c}
\Bigr]
\,d\nu(z).
\end{align}
We refer to $\Delta_k(b)$ as the residual drift term and to
$\Gamma_k(b)$ as the cross-term.

\paragraph{Residual drift bound.}
We first bound $\Delta_k(b)$.
By the definition of the residuals,
\begin{align}
\nonumber r(b,C,z;\eta^{0})
-
r(b,C,z;\hat{\eta}^{k})&
\\
\nonumber=
\langle
\bigl(
\mu_{W\mid C}^{0}(C)
-
\hat{\mu}_{W\mid C}^{k}(C)
\bigr)&\otimes\phi(z),\, b
\rangle\\
&+
\hat{p}^{k}(z\mid C)-p^{0}(z\mid C).
\end{align}
Applying the triangle inequality gives
\begin{align}
\nonumber &\hspace{-20pt}\bigl|
r(b,C,z;\eta^{0})
-
r(b,C,z;\hat{\eta}^{k})
\bigr|
\\
\nonumber &\le
\left|
\left\langle
\bigl(
\mu_{W\mid C}^{0}(C)
-
\hat{\mu}_{W\mid C}^{k}(C)
\bigr)\otimes\phi(z),\, b
\right\rangle
\right|\\
&\hspace{70pt}+
\bigl|
\hat{p}^{k}(z\mid C)-p^{0}(z\mid C)
\bigr|.
\end{align}
We next bound the inner-product term by Cauchy-Schwarz:
\begin{align}
\nonumber &\left|
\left\langle
\bigl(
\mu_{W\mid C}^{0}(C)
-
\hat{\mu}_{W\mid C}^{k}(C)
\bigr)\otimes\phi(z),\, b
\right\rangle
\right|
\\
&\le
\left\|
\bigl(
\mu_{W\mid C}^{0}(C)
-
\hat{\mu}_{W\mid C}^{k}(C)
\bigr)\otimes\phi(z)
\right\|_{\mathcal{H}_W\otimes\mathcal{H}_Z}
\,
\|b\|_{\mathcal{H}_W\otimes\mathcal{H}_Z}
\\
&=
\|\mu_{W\mid C}^{0}(C)-\hat{\mu}_{W\mid C}^{k}(C)\|_{\mathcal{H}_W}
\,
\|\phi(z)\|_{\mathcal{H}_Z}
\,
\|b\|_{\mathcal{H}_W\otimes\mathcal{H}_Z}.
\end{align}
Using Assumptions~\ref{bounded_feature_map}
and~\ref{bounded_bridge_class}, we obtain
\begin{align}
\nonumber &\hspace{-30pt}\bigl|
r(b,C,z;\eta^{0})
-
r(b,C,z;\hat{\eta}^{k})
\bigr|
\\
\nonumber &\le
K_Z B
\|\mu_{W\mid C}^{0}(C)-\hat{\mu}_{W\mid C}^{k}(C)\|_{\mathcal{H}_W}
\\
&\hspace{50pt}+\bigl|
p^{0}(z\mid C)-\hat{p}^{k}(z\mid C)
\bigr|.
\label{eq:stage1_pointwise_residual_bound}
\end{align}

Substituting \eqref{eq:stage1_pointwise_residual_bound} into the
definition of $\Delta_k(b)$ yields
\begin{align}
\Delta_k(b)
&\le
\int_{\mathcal{Z}}
\mathbb{E}
\Bigl[
\bigl(
A_k(C)+B_k(C,z)
\bigr)^2
\,\Big|\,
\mathcal{D}_{I_k^c}
\Bigr]
\,d\nu(z),
\end{align}
where, for brevity, we set 
\begin{align}
A_k(C)
&=
K_Z B
\|\mu_{W\mid C}^{0}(C)-\hat{\mu}_{W\mid C}^{k}(C)\|_{\mathcal{H}_W},
\\
B_k(C,z)
&=
\bigl|
p^{0}(z\mid C)-\hat{p}^{k}(z\mid C)
\bigr|.
\end{align}
Using the identity $(x+y)^2\le 2x^2+2y^2$, we conclude that
\begin{align}
\Delta_k(b)
&\le
2\Delta_{\mu,k}+2\Delta_{p,k},
\label{eq:stage1_drift_final}
\end{align}
where
\begin{align}
\nonumber&\Delta^{\mu}_{k}
=\\
&K_Z^2 B^2
\int_{\mathcal{Z}}
\mathbb{E}
\Bigl[
\|\mu_{W\mid C}^{0}(C)-\hat{\mu}_{W\mid C}^{k}(C)\|_{\mathcal{H}_W}^2
\,\Big|\,
\mathcal{D}_{I_k^c}
\Bigr]
\,d\nu(z),
\\
 &\Delta^p_{k}=\int_{\mathcal{Z}}
\mathbb{E}
\Bigl[
|p^{0}(z\mid C)-\hat{p}^{k}(z\mid C)|^2
\,\Big|\,
\mathcal{D}_{I_k^c}
\Bigr]
\,d\nu(z).
\end{align}

\paragraph{Cross-term bound.}
We next bound $\Gamma_k(b)$. We apply the conditional H\"older inequality with conjugate exponents $p=q=2$ to obtain
\begin{align}
\nonumber &\Gamma_k(b)
\le
\int_{\mathcal{Z}}
\Bigl(
\mathbb{E}
\bigl[
r(b,C,z;\eta^{0})^2
\bigr]
\Bigr)^{1/2}
\\
&\times
\Bigl(
\mathbb{E}
\Bigl[
\bigl(
r(b,C,z;\hat{\eta}^{k})
-
r(b,C,z;\eta^{0})
\bigr)^2
\,\Big|\,
\mathcal{D}_{I_k^c}
\Bigr]
\Bigr)^{1/2}
\,d\nu(z).
\end{align}
A second application of H\"older's inequality over the integral in $z$, again with conjugate exponents $p=q=2$, yields
\begin{align}
\Gamma_k(b)
&\le
L(b)^{1/2}\,
\Delta_k(b)^{1/2}.
\label{eq:stage1_cross_final}
\end{align}

Combining \eqref{eq:stage1_split},
\eqref{eq:stage1_drift_final}, and
\eqref{eq:stage1_cross_final}, we obtain
\begin{align}
\sup_{b\in\mathcal{B}}
\left|
L(b)-\tilde{L}_k(b)
\right|
&\le
2\sup_{b\in\mathcal{B}}
\Delta_k(b)
+
2\sup_{b\in\mathcal{B}}
L(b)^{1/2}\Delta_k(b)^{1/2}.
\end{align}
Hence, the Stage I term is controlled by the conditional errors of the nuisance estimators.

\subsection{Stage II residual}

In this subsection, we study the second term in \eqref{eq:stage_wise_split}.
\noindent By the definitions of $\tilde{L}^{\cf}$ and $\hat{L}^{\cf}$,
\begin{align}
\tilde{L}^{\cf}(b)-\hat{L}^{\cf}(b)
=
\frac{1}{K}\sum_{k=1}^K
\bigl(
\tilde{L}_k(b)-\hat{L}_k(b)
\bigr).
\end{align}
Hence, by the triangle inequality,
\begin{align}
\sup_{b\in\mathcal{B}}
\left|
\tilde{L}^{\cf}(b)-\hat{L}^{\cf}(b)
\right|
&\le
\frac{1}{K}\sum_{k=1}^K
\sup_{b\in\mathcal{B}}
\left|
\tilde{L}_k(b)-\hat{L}_k(b)
\right|.
\label{eq:stage2_fold_reduction}
\end{align}
It therefore suffices to control the foldwise deviation
$\sup_{b\in\mathcal{B}}|\tilde{L}_k(b)-\hat{L}_k(b)|$.

For each fold $k$ and each $b\in\mathcal{B}$, define the quadratic loss
\begin{align}
\ell_b^k(c,z)
:=
\left(
r(b,c,z;\hat{\eta}^{k})
\right)^2,
\qquad (c,z)\in\mathcal{C}\times\mathcal{Z}.
\end{align}
Let $Z^{\nu}$ be an auxiliary random variable with distribution $\nu$, independent of $C$ conditional on $\mathcal{D}_{I_k^c}$.
Then, by the definitions of $\tilde{L}_k$ and $\hat{L}_k$,
\begin{align}
\tilde{L}_k(b)
&=
\mathbb{E}
\Bigl[
\ell_b^k(C,Z^{\nu})
\,\Big|\,
\mathcal{D}_{I_k^c}
\Bigr],
\\
\hat{L}_k(b)
&=
\frac{1}{|I_k|M_k}
\sum_{i\in I_k}
\sum_{m=1}^{M_k}
\ell_b^k(c_i,\tilde z_m^{\,k}).
\end{align}
Therefore,
\begin{align}
\nonumber&\sup_{b\in\mathcal{B}}
\left|
\tilde{L}_k(b)-\hat{L}_k(b)
\right|
\\
&=
\sup_{b\in\mathcal{B}}
\Biggl|
\mathbb{E}
\Bigl[
\ell_b^k(C,Z^{\nu})
\,\Big|\,
\mathcal{D}_{I_k^c}
\Bigr]
-
\frac{1}{|I_k|M_k}
\sum_{i\in I_k}
\sum_{m=1}^{M_k}
\ell_b^k(c_i,\tilde z_m^{\,k})
\Biggr|.
\label{eq:stage2_uniform_deviation}
\end{align}

Conditional on $\mathcal{D}_{I_k^c}$, the nuisance estimate $\hat{\eta}^{k}$ is fixed.
Hence, the loss class
\begin{align}
\mathcal{L}_k
:=
\left\{
(c,z)\mapsto \ell_b^k(c,z)
:
b\in\mathcal{B}
\right\}
\end{align}
is deterministic on fold $k$.
Thus, \eqref{eq:stage2_uniform_deviation} is a standard uniform deviation term for a quadratic loss class.
This is exactly the same empirical-process object that appears in the second-stage analysis of \cite{hong2024model}, with the only difference being that the present argument is conditional on the auxiliary sample.

Accordingly, under the same boundedness and complexity assumptions imposed in \cite{hong2024model} for the corresponding second-stage loss class, each foldwise residual in \eqref{eq:stage2_fold_reduction} admits the same type of control.
Averaging over the folds then yields the corresponding bound for the full Stage II term.
Thus, the present decomposition does not introduce a new Stage II argument. Rather, it shows that, after conditioning on the auxiliary sample, the Stage II term falls within the same proof template as in \cite{hong2024model}. Cross-fitting changes only the conditioning structure, not the underlying empirical-process mechanism.

\end{document}